\documentclass[11pt,english]{thesis}
\usepackage{afterpage}
\usepackage{amsthm}

\DeclareMathOperator*{\argmax}{arg\,max}

\newcommand{\Api}{A_{\pi}}
\newcommand{\Apithold}{A_{\thold}}

\newcommand{\by}[1]{{\frac{1}{#1}}}
\newcommand{\cA}{\mathcal{A}}

\newcommand{\cP}{\mathcal{P}}
\newcommand{\cR}{\mathcal{R}}
\newcommand{\cS}{\mathcal{S}}

\newcommand{\diag}{\operatorname{diag}}

\newcommand{\Ea}[1]{\E\left[#1\right]}
\newcommand{\Eb}[2]{\E_{#1}\left[#2\right]}

\newcommand{\E}{\mathbb{E}}
\newcommand{\Epi}[1]{\E^\pi\left[#1\right]}
\newcommand{\Etilpi}[1]{\E^{\tilpi} \left[#1\right]}

\newcommand{\given}{|}

\newcommand{\indicator}{\mathbbm{1}}

\newcommand{\isd}{q}

\newcommand{\kl}[2]{D_{\rm KL}(#1 \ \| \ #2)}

\newcommand{\lrbrack}[1]{\left[#1\right]}

\newcommand{\maximize}{\operatorname*{maximize}}
\newcommand{\maxkl}{{\ensuremath D^{\rm max}_{\rm KL}}}

\newcommand{\meankl}[1]{{\ensuremath \overline D_{\rm KL}^{#1}}}

\newcommand{\normtwo}[1]{\left\lVert#1\right\rVert}

\newcommand{\pith}{\pi_{\theta}}

\newcommand{\Qpi}{Q_{\pi}}
\newcommand{\Qpis}{Q_{\pi^*}}
\newcommand{\Qpithold}{Q_{\thold}}

\newcommand{\Real}{\mathbb{R}}
\newcommand{\rhopi}{\rho_{\pi}}

\newcommand{\thold}{\theta_{\mathrm{old}}}

\newcommand{\tilpi}{\tilde{\pi}}

\newcommand{\Tpi}{\mathcal{T}_{\pi}}

\newcommand{\Vpi}{V_{\pi}}
\newcommand{\ViewLink}[1]{\href{#1}{\color{blue}{[View]}}}

\graphicspath{{figures/}}
\newcommand{\labelnotempty}[1]{
	\def\temp{#1}
	\ifx\temp\empty
	\else
		\label{#1}
	\fi
}
\newcommand{\singlefig}[4]{
	\begin{figure}[H]
		\centering
		\includegraphics[width=#1\columnwidth]{#2}
		\caption{#3}
		\labelnotempty{#4}
	\end{figure}
}
\newcommand{\subfig}[4]{ \includegraphics[width={#2}\columnwidth]{#1}\caption{#3}
	\labelnotempty{#4}
}

\renewcommand{\Return}{{\hspace{-12pt}\textbf{Return:}}}
\newcommand{\doublefig}[4]{

	\begin{figure}[h!]
		\centering
		\begin{subfigure}[t]{0.45\columnwidth}\centering #1 \end{subfigure}
		~
		\begin{subfigure}[t]{0.45\columnwidth}\centering #2 \end{subfigure}
		\caption{#3}
		\labelnotempty{#4}\end{figure}
}

\newcommand{\triplefig}[5]{

	\begin{figure}[ht!]
		\centering
		\begin{subfigure}[h!]{0.30\columnwidth}\centering #1 \end{subfigure}
		~
		\begin{subfigure}[h!]{0.30\columnwidth}\centering #2 \end{subfigure}
		~
		\begin{subfigure}[h!]{0.30\columnwidth}\centering #3 \end{subfigure}
		\caption{#4}
		\labelnotempty{#5}
	\end{figure}
}

\newcommand{\blankpage}{%
	\null
	\thispagestyle{empty}%
	\addtocounter{page}{-1}%
	\newpage
}

\newcommand{\acro}[2]{\newacronym{#1}{#1}{#2}}
\usepackage{palatino}
\acro{A3C}{Asynchronous Actor-Critic Agents}
\acro{CPI}{Conservative Policy Iteration}
\acro{DPG}{Deterministic Policy Gradient}
\acro{DP}{Dynamic Programming}
\acro{HREPS}{Hierarchical Relative Entropy Search}
\acro{HRL}{Hierarchical Reinforcement Learning}
\acro{KL}{Kullback Leibler}
\acro{LSPI}{Least Square Policy Iteration}
\acro{MDP}{Markov Decision Process}
\acro{MI}{Mutual Information}
\acro{PVF}{Proto-Value Function}
\acro{PI}{Policy Iteration}
\acro{REPS}{Relative Entropy Search}
\acro{RIM}{Regularised Information Maximization }
\acro{RL}{Reinforcement Learning}
\acro{RIMs}{Regularization via Information Maximization}
\acro{SARSA}{State-Action-Reward-State-Action}
\acro{SL}{Supervised Learning}
\acro{SSL}{Semi-Supervised Learning}
\acro{SpIn}{Spectral Inference Network}
\acro{TRPO}{Trust Region Policy Optimization}
\acro{UL}{Unsupervised Learning}
\acro{VI}{Value Iteration}
\acro{h-DQN}{hierarchical Deep Q Learning}

\newtheorem{theorem}{Theorem}[chapter]
\newtheorem{lemma}{Lemma}[chapter]
\makeglossaries

\begin{document}
\pagenumbering{roman}
\begin{center}
{\LARGE \scshape{ \Large Reinforcement Learning with Options and State Representation}
} 
		
\vspace{2cm}
\line(1,0){200} \\

\large \scshape{
Ayoub GHRISS\\}

\vspace*{0.3cm}
Supervisor: \\
Dr Masashi Sugiyama\\
RIKEN AIP\\
\vspace*{0.3cm}
Academic Supervisor:\\
Dr Alessandro Lazaric\\
ENS Paris-Saclay	

\vspace*{0.5cm}
		
\line(1,0){200}
		
\vspace{1cm}
\textsc{
Master Thesis
\\
Mathématiques, Vision, Apprentissage}

\vspace{2cm}
\textnormal{	Submitted to Department of Applied Mathematics, ENS Paris-Saclay
\\
Performed in RIKEN Advanced Intelligence Project 
}
		
\line(1,0){200} \\	
\vspace*{1cm}	\textnormal{
Tokyo, Japan\\
Sept, 2018}
	\thispagestyle{empty}
	\end{center}

\newpage

\afterpage{\blankpage}\newpage
\addcontentsline{toc}{chapter}{Acknowledgments}

\chapter*{Acknowledgments}

A very special gratitude goes out to all researchers and administration staff in
RIKEN AIP for making the amazing time in RIKEN labs a memorable, enriching, and
scientifically dense experience.\par

I am grateful to Dr Masashi Sugiyama for the opportunity he gave me to join his
Imperfect Information Learning team to perform my internship.\par

With a special mention to Voot Tangkaratt for the concise and straight to point
responses to my inquiries, Charles Riou for his expertise and reviews on the best
Japanese raustaurants in Tokyo, Jill-Jenn Vie for the enthusiastic guiding and
knowledge about Japan, and Imad El-Hanafi for sharing his unlimited arsenal of
Moroccan humour.

\vspace{2em}
Thank you for all your encouragement!

\afterpage{\blankpage}
\addcontentsline{toc}{chapter}{Abstract}

\chapter*{Abstract}

The current thesis aims to explore the reinforcement learning field and build on
existing methods to produce improved ones to tackle the problem of learning in
high-dimensional and complex environments. It addresses such goals by decomposing
learning tasks in a hierarchical fashion known as Hierarchical Reinforcement
Learning.

We start in the first chapter by getting familiar with the Markov Decision Process
framework and presenting some of its recent techniques that the following chapters
use. We then proceed to build our Hierarchical Policy learning as an answer to the
limitations of a single primitive policy. The hierarchy is composed of a manager
agent at the top and employee agents at the lower level. \par

In the last chapter, which is the core of this thesis, we attempt to learn
lower-level elements of the hierarchy independently of the manager level in what is
known as the "Eigenoption". Based on the graph structure of the environment,
Eigenoptions allow us to build agents that are aware of the geometric and dynamic
properties of the environment. Their decision-making has a special property: it is
invariant to symmetric transformations of the environment, allowing as a
consequence to greatly reduce the complexity of the learning task.

\tableofcontents
\pagenumbering{arabic}
\newpage

\chapter{Reinforcement Learning Framework} \label{chap:chapter1}

\begin{abstract}

	We introduce the Markov Decision Process paradigm, some approximate Dynamic
	Programming methods, their counterparts in continuous space states, and lay
	their fundamentals, which will be used in the next chapters.

\end{abstract}

\section*{Introduction}

Reinforcement Learning (\citeauthor{rlIntroSutton} \citep{rlIntroSutton}) is a
mathematical framework for active learning from experience. The concept is based on
developing agents capable of learning in an environment through trial-and-error.
The environment consists of a set of states $S$ and actions $A$. The agent
interacts with the environment by moving from one state to another and receives
feedback about its decisions through rewards. The goal of the agent is to learn a
decision-making process, also called a policy, that maximizes the expected
accumulated reward.\par

Based on this definition, we can already distinguish between two different types of
Reinforcement Learning, depending on the prior knowledge about the environment:
model-based and model-free. In the first type of RL, the dynamics of the
environment are known prior to the learning process, and the agent incorporates the
knowledge of the transition process between states (the laws of physics in a
robotics task, for example). In contrast, in the second type, the agent has no
prior knowledge about the environment dynamics and needs to acquire it while
learning the best policy. In this work, we address model-free RL.

In this chapter, we introduce the Markov Decision Process paradigm (\textbf{MDP})
and summarize its fundamentals.

\section{Markov Decision Process}
In standard RL models, the agent is represented as a time-homogeneous \Gls{MDP}. An
\Gls{MDP} is a tuple $ \Omega =(\cS,\cA,\cR,\cP) $ where $\cS$, $\cA$ are
respectively the state and action spaces, $\cP$ is the transition kernel where
$P(s_{t+1}|s_t, a_t)$ represents the probability of transition from state $s_t$ to
a new state $s_{t+1}$ when taking the action $a_t$. Each decision of taking action
$a_t$ at state $s_t$ can yield a certain reward $R(s_t,a_t)$ which we assume is
bounded and deterministic given $s_t$ and $a_t$. In the case of an infinite
horizon, future rewards can be discounted with a factor $\gamma$ and the \Gls{MDP}
becomes $ \Omega =(\cS,\cA,\cR,\cP,\gamma)$, which will be our \Gls{MDP} of
interest in this work.

\singlefig{0.35}{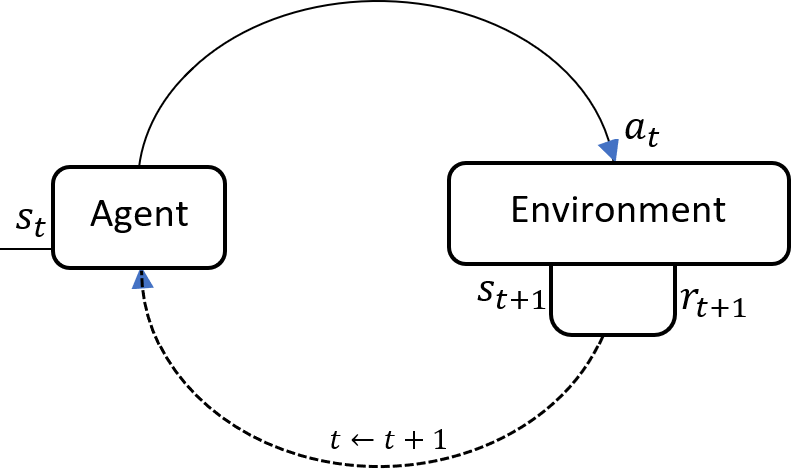}{MDP structure}{fig:MDP1}


The goal in \gls{RL} is to permit the agent to learn autonomously, through trial
and error, a policy $\pi: \cS \times \cA \rightarrow [0,1]$ that allows to maximize
the expected accumulated rewards $\eta(\pi)$ defined as:

\begin{align}
	\eta(\pi) = \Epi{\sum_{t=0}^{\infty} \gamma^t r(s_t)}\text{,}
\end{align}

where $\E^\pi$ means:

\begin{itemize}
	\item $ s_0 \sim \rho_0$, $\rho_0$ is the initial distribution over the state space
	\item The action is then distributed following $a_t \sim \pi(.|s_t) $.
	\item The next states follow $ s_{t+1} \sim P(.|s_t, a_t)$
\end{itemize}

The time homogeneity implies that at any time $t$ we have:
\[
	P(s_{t+1}=s'|s_t=s, a_t=a) = P(s'|s, a)
\]

\subsection{Value Functions}

Value functions measure the expected return conditioned on starting state and/or
actions: the state-action value function $\Qpi : \cS\times\cA \rightarrow \Real$,
and state value function $\Vpi : \cS \rightarrow \Real$ both defined with respect
to policy $\pi$:

\begin{align}
	\Qpi(s, a) & = \Eb{s_{t},a_{t},\dots}{\sum_{t=0}^{\infty} \gamma^{t} r(s_{t})|s_0=s,a_0=a}\text{,} \\
	\Vpi(s)    & = \Eb{a_t,s_t,\dots}{\sum_{t=0}^{\infty} \gamma^t r(s_{t})|s_0 = s}\text{,}
\end{align}

and define the advantage function:
\[
	\Api(s, a) = \Qpi(s,a) - \Vpi(s)
\]

The value $\Qpi(s, a)$\label{par:rhoadv} is the expected accumulated reward when
taking the action $a$ at state $s$. When we sample this action $a$ following the
policy $\pi(.|s)$, then we expect to accumulate $\Vpi(s)$. The advantage $\Api(s,
a)$, on the other hand, translates the gain of taking the action $a$ in state $s$
instead of following the policy $\pi$. Consequently, the policy is optimal if the
advantage is non-positive at each state for any adversary action $a$.

\subsection{MDP Properties}

Each policy induces a distribution over the state space which is proportional to
the "unnormalized" state density $\rhopi$ defined as
\begin{align}
	\rhopi(s) & = \Epi{\sum_{t=0}^{\infty} \gamma^t \indicator_{s_t=s}} \nonumber                                                                \\
	          & =\E_{s_0\sim \rho_0,a_t \sim \pi(.|s_t)}{\sum_{i=0}^{\infty} \gamma^{t}P(s_{t+1}=s|s_t, a_t)} \label{eq:discountdensity}\text{.}
\end{align}

Using $\rhopi$, the expected return $\eta(\pi)$ can be formulated as:

\begin{align}
	\eta(\pi) = \int_\cS \int_\cA \rhopi(s)\pi(a|s)r(s,a)dsda
	\label{eq:etarho}
\end{align}

We can also simplify $\Vpi$ using the initial density $\rho_0$:
\begin{align}
	\eta(\pi) = \E_{s_0\sim \rho_0}{\Vpi(s_0)}\text{,}
\end{align}

This simple property can be used to prove the following important lemma, which
links the discounted density with and the expected return of two policies (see
\cite{Kakade02} for proof):
\begin{lemma}\label{eq:mus}
	Let $\pi$ and $\tilpi$ be two policies over the states space $\cS$, we have the following
	\begin{align}
		\eta(\pi) = \eta(\tilpi)+\Etilpi{\sum_t \gamma^t\Api(s_t, a_t)} ,
	\end{align}
	and using the density $\rhopi$ :
	\begin{align}
		\eta(\pi) = \eta(\tilpi) + {\int_\cS^{}} {\int_\cA^{}} \rho_{\tilpi}(s)\tilpi(a|s)\Api(s_t,a_t)\label{eq:rhoadvantage}
	\end{align}
\end{lemma}

This lemma, as discussed in subsection~\ref{par:rhoadv}, implies that a policy
update $\pi\leftarrow \tilpi$ with non-negative expected advantage over $\pi$ at
every state will increases $\eta(\pi)$.

\subsection*{Remarks}

In the remaining study, we restrict our interest to the following \Gls{MDP}s

\begin{itemize}
	\item $r :\cS \times \cA \rightarrow \Real$ is deterministic given the state and action
	\item The transition mechanism from $s_t$ to $s_{t+1}$ given $a_t$ is
	      deterministic
\end{itemize}

While deterministic rewards apply to a large set of environments, using a prior
distribution on the reward has recently been proven to improve many of the
reinforcement learning models, even when the reward is deterministic
\cite{DistributionRL}. It has also been proven that in distributional \Gls{RL},
some of the properties of \Gls{MDP} change, such as the contraction property of the
Bellman operator, which we'll define in the following section.

\section{Dynamic Programming}
\gls{DP} \cite{bellman1954} is mathematical optimization paradigm developed
by Richard Bellman in the 1950s. It mainly specializes in optimization problems where
the sub-problems can be nested recursively inside the general one. Using Bellman
equations with specific assumptions on the objective function, we are guaranteed to
have a link between the general solution and the sub-problems solutions. In our
case, we can notice that for two consecutive states $s_t, s_{t+1}$, if a policy
$\pi$ maximizes $\Vpi(s_t)$ then it necessarily maximizes $\Vpi(s_{t+1})$.

\subsection{Bellman Equations}\label{sect:bellmann}

The Bellman equations applied to the value functions can be written as follows:
\begin{align}
	\Vpi(s)   & = \Eb{a\sim \pi(.|s)}{r(s,a) + \gamma \Eb{s'\sim P(.|s, a)}{\Vpi(s')}} \\
	\Qpi(s,a) & = r(s,a) + \gamma \Eb{s'\sim P(.|s, a)}{\Vpi(s')}.
\end{align}\par
\noindent We define the Bellman operators $\Tpi:\Real^\cS\rightarrow\Real^\cS$ as:
\begin{align}
	\Tpi(f)(s) = \Eb{a\sim \pi(.|s)}{r(s,a) + \gamma \Eb{s'\sim P(.|s, a)}{f(s')}},
\end{align}
and the optimal Bellman operator $\mathcal{T^*}: \Real^{\cS \times \cA} \rightarrow \Real^{\cS \times \cA}$ as:
\begin{align}
	\mathcal{T^*}(f)(s) = \max_a \left[r(s,a) + \gamma \Eb{s'\sim P(.|s, a)}{f(s')}\right]
\end{align}

Both $\Tpi$ and $\mathcal{T^*}$ are contractions with Lipschitz constant $\gamma$.
Hence, both admit unique fixed points and we can conclude that:

\begin{itemize}
	\item Every policy $\pi$ defines unique value functions $\Vpi$ and $\Qpi$
	\item There is unique optimal value function $V^*$, the fixed point of
	      $\mathcal{T^*}$, which verifies: $$ V^*(s) = \max_a\Eb{\pi(.|s)}{r(s,a) +
	      \gamma \Eb{s'\sim P(.|s, a)}{V^*(s')}} $$
\end{itemize}

The optimal action value function $Q^*$ can be induced:
\begin{align}
	Q^*(s,a) & = r(s,a) + \gamma V^*(s) \nonumber                                          \\
	         & = r(s,a) + \gamma \max_a\Eb{s'\sim P(.|s, a)}{\max_{a'}Q^*(s',a')}\nonumber
\end{align}

These properties, however, do not imply the uniqueness of the optimal/greedy policy
$\pi^*$ associated with the optimal functions:
\[
	\pi^*(a|s) = \mathbbm{1}_{\argmax_a \Qpis(s,a)}
\]

Based on the advantage function $A^*$, if two actions $a_1$ and $a_2$ at a certain
state $s$ verify $A^*(s,a_1)=0$ and $A^*(s,a_2)=0$, then we are indifferent between
the two as they both have the same future accumulated discounted rewards.

\subsection{Iterative Learning}

Based on the Bellman equations and the properties of the operators, we introduce in
this section different methods used to learn the optimal value functions and/or the
optimal policies. The iterative learning is used for finite \Gls{MDP}, but they
share many fundamentals with their counterparts in continuous state \Gls{MDP}. It
is important to note that the presented methods belong to the approximate Dynamic
Programming, where we attempt to approximate value functions and policies. The
approximation quality depends on the richness of the family of functions on which
we optimize (least-square, neural networks...).

As noted in \autoref{sect:bellmann}, since the Bellman operators are contraction
mappings, we have the following result:

\begin{theorem}[Contraction mapping theorem]
	\label{ContractionTh}

	Let $\mathcal{X}$ be a complete metric space, and $F: \mathcal{X}\rightarrow
	\mathcal{X}$ be a contraction. Then $F$ has a unique fixed point, and under the
	action of iterates of $F$, all points converge with exponential speed to it.

\end{theorem}

Using Bellman operators, by starting from any random function $f_0:\cS\rightarrow
\Real$ and defining $f_n = \mathcal{T}^n(f_0)$, then $(f_n)_n$ converges to the
fixed point of $\mathcal{T}$. While this applies to both V and Q functions, it is
generally preferable to learn the action value $Q$ (or the advantage $A$) function
since they allow inducing the optimal action at each state. Learning only $V$,
however, requires knowledge of either the reward or the temporal difference
$V(s_{t+1})-V(s_t)$, implying prior knowledge about the transition mechanism to
predict the next state.

\subsubsection{Value Functions Iteration}
Using the contraction mapping property, we can iteratively approximate the value
function $V$ in a discrete infinite time horizon \Gls{MDP}:
\begin{algorithm}
	\label{alg:VIter}
	\caption{Value Iteration Algorithm}
	\begin{algorithmic}
		\Require{An initial value function $V : \cS \rightarrow \Real$ }
		\For{$t \gets 1 \textrm{ to } T $}
		\State{$s_0 \sim \rho_0$}
		\State{Sample a trajectory $\tau =(a_0,s_1,s_2...s_n)$}
		\State{Estimate $V_{t+1} \gets T^*(V_t)$ based on $\tau$}
		\EndFor\\

		\Return{ The greedy policy associated with $V_T$}
	\end{algorithmic}
\end{algorithm}

The same algorithm can be applied to the state-action function $Q_t$. The updates
can be performed in an asynchronous fashion by updating the value function at each
step of the trajectory. The trajectory sampling can be done in an on-policy manner
by using the greedy policy $\pi_t$ at iteration $t$, or in an off-policy way. The
on-policy method can imply a risk of staying in a local domain of the state space.
The common practice is to add stochasticity with an $\epsilon-$greedy method by
choosing random actions with probability $\epsilon$. The contraction mapping
property guarantees convergence only if all states can be visited infinitely often.
This is known as the exploitation-exploration dilemma in \Gls{RL}, since we want to
exploit the learned knowledge while being able to explore new states once in a
while.

\subsubsection{Policy Iteration}

In a similar way to \Gls{VI}, \Gls{PI} starts with an initial policy $\pi_0$ then
alternates between learning $\Vpi$ and $\pi$.

\begin{algorithm}
	\label{alg:PI}
	\caption{Policy Iteration Algorithm}
	\begin{algorithmic}
		\Require{An initial value policy $\pi_0$ }
		\For{$t \gets 1 \textrm{ to } T $}
		\State{Sample a trajectory $\tau =(s_0,a_0,s_1,s_2...s_n)$}
		\State{Evaluate ${\Vpi}_t$}
		\State{$\pi_{t+1} \gets {\pi^*}_t$ with ${\pi^*}_t$ the greedy policy of ${\Vpi}_t$}
		\State{Estimate $V_{t+1} \gets T^*(V_t)$ based on $\tau$}
		\EndFor\\
		\Return{ The learned policy $\pi_{T+1}$}
	\end{algorithmic}
\end{algorithm}

The same thing about exploration-exploitation applies to \Gls{PI}. The \Gls{PI} and
\Gls{VI} methods assume that we can efficiently estimate $V_{\pi,t}$ and $\pi_t$,
which are practically challenging, since the value functions are formulated as the
expectations of the policy and initial distribution. Practically, we sample one or
several trajectories and estimate the expectation over the trajectories and
incrementally update the previous quantities.

\subsubsection{Temporal Difference}

When starting from a state $s_t$ and taking action $a_t$ to observe $s_{t+1}$, the
one-step temporal difference error is defined as:
\begin{align}
	\Delta_t \vcentcolon= r(s_t,a_t) + \gamma {\Vpi}^{(n)}(s_{t+1}) - {\Vpi}^{(n)}(s_t),
\end{align}

with ${\Vpi}^{(n)}$ being the value function estimation at iteration $n$ and
$r(s_t,a_t) + \gamma {\Vpi}^{(n)}(s_{t+1})$ the TD target.

For $\lambda \in [0,1]$, the $\lambda$-temporal difference TD($\lambda$) is used to
update the estimated function: $$ {\Vpi}_{n+1}(s_t) \gets {\Vpi}_n + \alpha_n
\sum_{s=0} \lambda^s\Delta_s. $$ The parameter $\lambda \in [0,1]$ is called the
trace decay. A higher $\lambda$ reflects lasting reward traces, with $\lambda = 1$
the value function update is performed at the end of the trajectory, while for
$\lambda = 0$ the update is performed after each decision step. The update rate
$(\alpha_i)$ is generally chosen to satisfy the Robbins-Monro condition:

\begin{align}
	\sum_i \alpha_i   & = \infty \\
	\sum_i \alpha_i^2 & < \infty
\end{align}

Given enough samples, the Robbins-Monro condition guarantees that the
approximations $({\Vpi}^{(n)})_n$ converge almost surely to the true $\Vpi$.

The temporal difference error plays an important role in continuous and large space
state \Gls{MDP}, and it is largely used for $Q$ and $V$ function learning
(SARSA\cite{SARSA}, Q-learning\cite{rlIntroSutton}). Its main advantage is that it
can be applied in high-dimensional state space \Gls{MDP} using differentiable
approximations. In the following section, we will be exposing some of these
methods.

\section{Gradient Methods}\label{sect:PGM}
In continuous and high-dimensional state space MDP, Temporal Difference is the core
of several approximate \gls{DP} models. The most basic methods use parameterized
functions (Least Squares, Neural Networks) and attempt to update the parameters to
minimize $TD(\lambda)$ using closed-form solutions or gradients depending on the
parameterization. Neural Networks have gained popularity for their ease of use in
gradient descent methods and their ability to scale \Gls{RL} techniques to
high-dimensional state space.\par

In this section, we go beyond the simple applications of $TD(\lambda)$ to briefly
present a particular class of gradient methods: policy gradient methods. We
summarize the work of \citeauthor{NeuPGM}. In their work on unifying policy
gradient methods, they reintroduce 3 popular policy gradient methods from a convex
optimization point of view:

\begin{itemize}
	\item \gls{REPS} (\citeyear{REPS})
	\item \gls{TRPO} (\citeyear{TRPO})
	\item \gls{A3C} (\citeyear{A3C})
\end{itemize}

They demonstrate that these methods aim to solve the same objective using different
constraints.

\subsection{Regularization functions}

We define the Shannon entropy $E_S$ for a policy $\pi$ as

\begin{align}
	E_S(\pi) \vcentcolon= \sum_{s,a} \pi (x,a) \log \pi (x,a),
\end{align}

and the negative conditional entropy $E_c$ :

\begin{align}
	E_C(\pi) \vcentcolon= \sum_{s,a} \pi (s,a) \log \dfrac{\pi (s,a)}{\sum_{b} \pi (s,b)}.
\end{align}

To constrain the policy updates, we use the Bregman divergences $D_S$ and $D_C$
associated with the entropies $E_S$ and $E_D$, respectively, defined as:

\begin{align}
	D_S(\pi \| \tilpi) \vcentcolon= \sum_{s,a} \pi (s,a) \log \dfrac{\pi (s,a)}{\tilpi (s,a)} \\
	D_C(\pi \| \tilpi) \vcentcolon= \sum_s \pi(s) \sum_a \pi (a|s) \log \dfrac{\pi (a|s)}{\tilpi (a|s)}
\end{align}

$E_S$ and $E_C$ are two strictly convex functions over the space of distributions
on the $\cS \times \cA$. $D_S$ is in fact the Kullback-Leibler divergence while
$D_C$ is known as the conditional KL divergence over the state space.

\subsection{Convex optimization perspective}

With regularization functions defined above, \citeauthor{NeuPGM} proved that at the
$t$ gradient update of the policy, all the three methods attempts to solve the
following objective :

\begin{align}\label{eq:gmobject}
	\pi_{t+1} = \argmax_\pi \Bigl ( \eta(\pi) - \alpha. D(\pi \| \pi_{t}) \Bigr ),
\end{align}

where D being the divergence associated with the method.

\paragraph{\gls{REPS}} uses mirror descent with Bregman divergence $D_S$ in \ref{eq:gmobject}.
\[
	\pi_{t+1} = \argmax_\pi \Bigl ( \eta(\pi) - \alpha. D_S(\pi \| \pi_{t}) \Bigr ).
\]

In each update $t$, the trajectories are sampled using the current policy $\pi_t$.
Using Lagrangian strong duality and the strict convexity of the entropy $E_S$,
\gls{REPS} is equivalent to the mirror descent algorithm using the Bregman
divergence $D_S$. Consequently, \gls{REPS} converges to the optimal policy.
However, as is the case for all other methods, the optimality is relative to the
class of parameterized functions on which we optimize the objective in
(\ref{eq:gmobject}).

\paragraph{\gls{TRPO}} solves \ref{eq:gmobject} using the Bregman divergence $D_C$.
\gls{TRPO} formulates the problem as:

\begin{align}
	\label{eq:TRPOobj}
	\begin{aligned}
		 & \underset{x}{\text{maximize}}
		 &                               & \eta(\pi)                       \\
		 & \text{subject to}
		 &                               & D_C(\pi \| \pi_{t})\leq \delta,
	\end{aligned}
\end{align}

where $\delta$ is a hyper-parameter controlling the update step. \gls{TRPO}
determines the trust region update using conjugate gradient to bound the
divergence.

\paragraph{\Gls{A3C}}
accumulates gradients of the objective (\ref{eq:gmobject}) with regularization
$D_C$. The gradients are provided asynchronously from several agents. Instead of
using mirror descent, it uses the dual averaging method by incorporating updates
provided by the agents. However, in their unified framework, \citeauthor{NeuPGM}
argue that \Gls{A3C} does not present any convergence guarantees. The interesting
result is that the optimal update in the case of \Gls{A3C} has a closed form:

\begin{align}
	\pi_{t+1} \propto \exp{\frac{A_{\pi_k}}{\beta_k}},
\end{align}
where ($\beta_k$) an increasing sequence to insure convergence.

\subsection{Trust Region Policy Optimization}

\Gls{TRPO} optimizes the objective (\ref{eq:TRPOobj}) by first replacing the
discounted state density $\rho_{\tilpi}$ in \autoref{eq:rhoadvantage} with
$\rho_{\pi}$. The obtained the quantity is:

\begin{align}\label{eq:adv0}
	L_{\pi}(\tilpi) = \eta(\pi) + \sum_s \rho_{\pi}(s) \sum_a \tilpi(a \given s) \Api(s,a).
\end{align}

$L_{\pi}(\tilpi)$ is a first-order approximation of $\eta(\tilpi)$.
At this stage, we consider parameterized policies $(\pith)_{\theta \in \Real^n}$
and we use $\theta$ to design $\pith$. \citeauthor{TRPO} proved the following
inequality, analogous to the \gls{CPI}:

\begin{lemma}
	\begin{align}
		\eta(\theta) \ge L_{\thold}(\theta)-C \maxkl(\thold, \theta).
	\end{align}
\end{lemma}

Maximizing the right side of the inequality improves (remains constant in the worst
case) the surrogate. \Gls{TRPO} proceeds by relaxing the bound based on
$D_{KL}^{max}$ and the objective becomes:

\begin{align}
	 & \maximize_{\theta} L_{\thold}(\theta) \label{eq:trprob}                            \\
	 & \text{\ \ subject to } \meankl{\rho_{\thold}}(\thold,\theta) \le \delta. \nonumber
\end{align}

$\maxkl$ is the maximum KL divergence between policies over the sampled states,
while $\meankl{}$ is the mean KL divergence.

\subsubsection{\Gls{TRPO} recipe}

By simulating trajectories following the current policy, we estimate the objective
and the bound:
\begin{itemize}
	\item Expand $L_{\thold}$ in Equation~\eqref{eq:trprob}, we replace $\sum_s
	      \rho_{\thold}(s) \lrbrack{\dots}$ in the objective by the expectation
	      $\frac{1}{1-\gamma}\Eb{s \sim \rho_{\thold}}{\dots}$
	\item Replace $\Apithold$ by the $Q$-values $\Qpithold$ in
	      Equation~\eqref{eq:adv0},
	\item replace the sum over the actions by an importance sampling estimator
	      based on ${\pith}_{old}$.
\end{itemize}

\Gls{TRPO} problem is hence equivalent to the following one:

\begin{align}
	 & \maximize_{\theta} \Eb{s \sim \rho_{\thold}, a \sim \isd}{ \frac{\pith(a\given s)}{{\pith}_{old}(a\given s)} \Qpithold(s,a)} \label{eq:trprobexp} \\
	 & \text{\ \ subject to }
	\Eb{s \sim \rho_{\thold}}{\kl{\pi_{\thold}(\cdot \given s)}{\pi_{\theta}(\cdot\given s)}}
	\le \delta. \nonumber
\end{align}
\subsubsection{TRPO Algorithm:}
\begin{algorithm}
	\caption{Trust Region Policy Optimization}
	\begin{algorithmic}
		\Require{A parameterized policy $\pi_\theta$ and $\theta$=$\theta_0$ }
		\For{$i \gets 1 \textrm{ to } T $}
		\State{Run policy for N trajectories}
		\State{Estimate advantage function at all time steps}
		\State{Compute objective gradient $g_\theta$}
		\State{Compute A}
		\State{Use Conjugate Gradient to compute $\beta$ and $s$}
		\State{Compute the rescaled update line search}
		\State{Apply update to $\theta$}
		\EndFor
	\end{algorithmic}
\end{algorithm}

We denote by $g_\theta$ the gradient of the objective with respect to $\theta$. We
write the quadratic approximation of the constraint function :

\[
	\meankl{}(\thold,\theta) \approx \frac{1}{2} (\theta-\thold)^T A (\theta-\thold),
\]

where
\[
	A_{ij} =\frac{\partial}{\partial\theta_i}\frac{\partial}{\partial\theta_j} \meankl{}(\thold,\theta).
\]

Since $A$ is symmetric semi-definite, we use the conjugate gradient method to solve
$s=A g_\theta$. The solution $s$ is the search direction. We then obtain the
maximal step scale $\beta = \sqrt{2 \delta / s^T A s}$ that verifies $\delta =
\meankl{} \approx \frac{1}{2} (\beta s)^T A (\beta s) = \frac{1}{2}\beta^2 s^T A
s$. The last step is the line search, which allows us to shrink the step
exponentially until the objective improves.

\subsection{Experiment}

We test the TRPO algorithm in the "4-room" environment. The motivation for such a
choice of environment is the ability to compare results between continuous space,
where we use the raw RGB screen as input, and the exact solutions in the tabular
case. The agent (white square) goal is to reach the red square to receive a $+1$
reward, in the minimal number of steps. The horizon is virtually infinite, but
practically sufficient to explore all states before episodes end. We can also
expand the state space by expanding the grid's size. For a size: $n = 36$ and
randomly located white and red squares, the total number of states is of the order
of $10^6$. We can also control the sparsity by varying the number of red squares.

\doublefig{\subfig{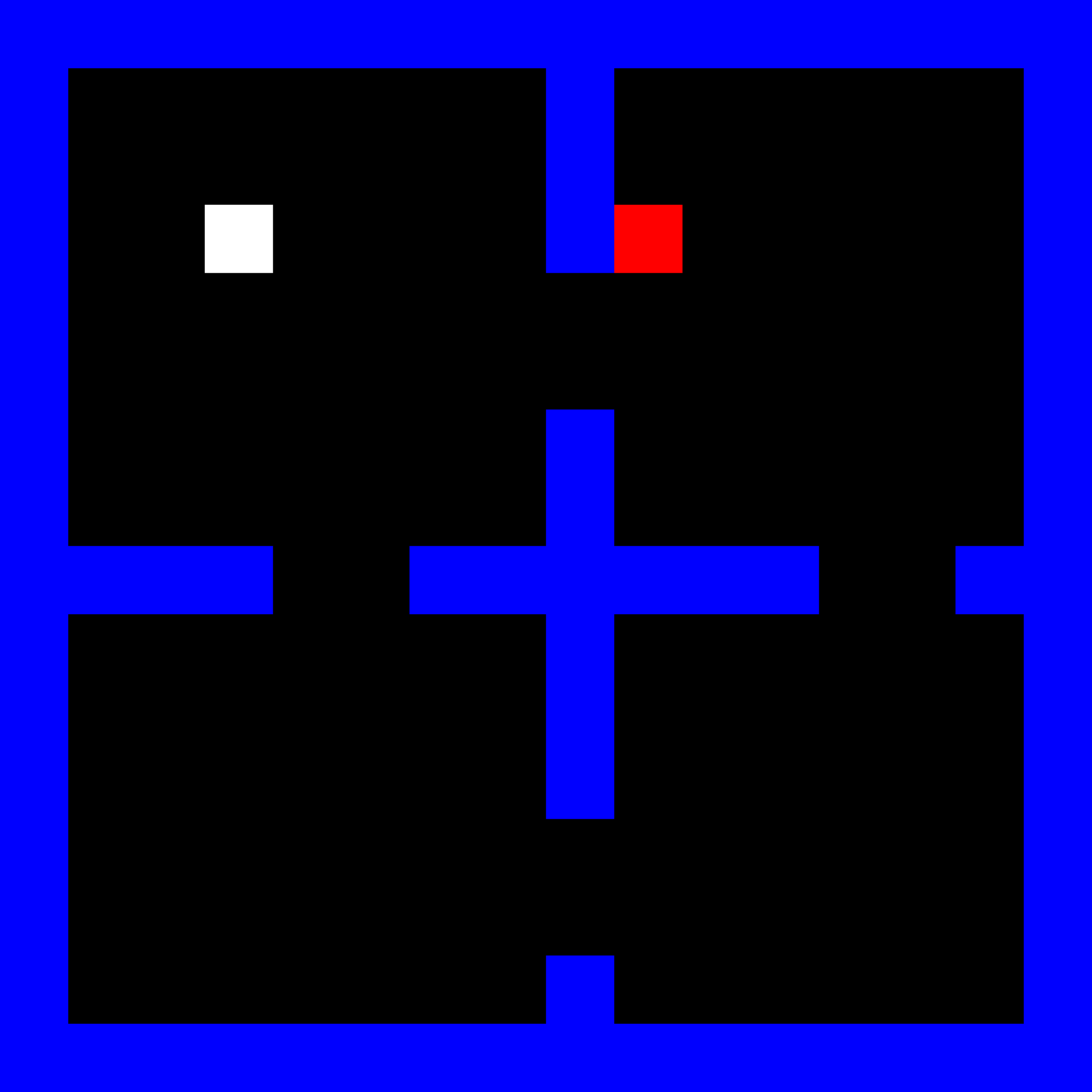}{0.5}{$n=16$}{fig:env16}}{
	\subfig{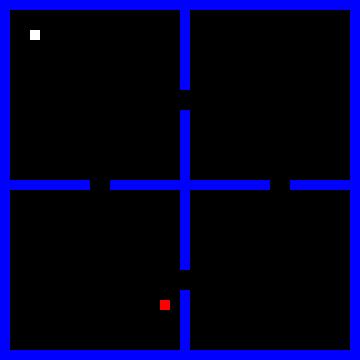}{0.5}{$n=36$}{fig:env}
}{4-room environment for different sizes}{fig:roomenv}

Our implementation of the TRPO agent is a PyTorch adaptation of the official
TensorFlow implementation from OpenAI baselines \citep{baselines}. We've tested the
implementation on some of the Atari Games, and the scores were similar to the ones
reported in \cite{TRPO}. However, for the 4-room environment with a large size in
the stochastic framework (the target red square is randomly located), the TRPO
fails to improve the policy. When reaching the goal, it quickly forgets the
experience. This is mainly due to the sparsity of the rewards.

We plot below the mean time to the goal for each room size. We start from the
uniform policy. The plots are moving averages of 5 runs. We've also run the
algorithm stagnates at the maximum episode length.

\singlefig{0.5}{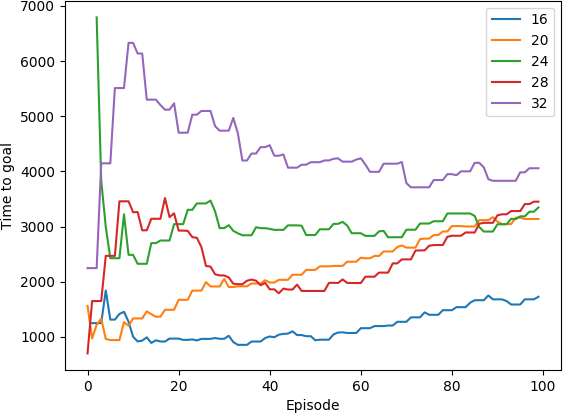}{TRPO in 4 Rooms}{fig:trpo1}

\section*{Conclusion}

After presenting the fundamentals of Reinforcement Learning in a multi-state space,
we summarized \citeauthor{NeuPGM}'s unifying framework for entropy-regularized
policy gradient methods. We tested the TRPO algorithm on a few environments and
demonstrated some of its limitations in a sparse environment.

In the following chapter, we attempt to tackle these limitations in the framework
of the \gls{HRL} by extending the \gls{TRPO} to incorporate task hierarchies.

\chapter{Hierarchical Reinforcement Learning}
\begin{abstract}

	We introduce \Gls{HRL} as a way to tackle the learning scaling problem and
	sparsity. We review some recent approaches used to learn decision hierarchies,
	and conclude with an on-policy hierarchical reinforcement learning method based
	on the work of \citeauthor{TakayukiHRL}~\cite{TakayukiHRL} for an off-policy .

\end{abstract}

\section*{Introduction}

As is the case with human learning, biological organisms can master tasks from
extremely small samples. The fact that a child can acquire a visual object concept
from a single unlabeled example is evidence for preexisting cortical mechanisms
that facilitate such efficient learning \cite{HLThesis}. While reinforcement
learning is rooted in Neuroscience and Psychology \cite{HRLNeuro}, Hierarchical
Reinforcement Learning was developed in the machine learning field based on the
abstraction of either the states or the actions.

State abstraction \cite{StateAbstract1,StateAbstract2} is based on learning a state
representation concentrating on features relevant to the decision process and
excluding the irrelevant ones; hence, behaviorally similar states should induce the
same representation. On the other hand, action abstraction is based on using
temporally extended policies \cite{OptionsSutton} to scale the learning task. Once
a temporally abstract action is initiated, the execution of its associated policy
continues until a set of specified termination states is reached. Thus, the
selection of a temporally abstract action ultimately results in the execution of a
sequence of actions.

In this chapter, we focus on temporal action abstraction for a continuous state
\gls{MDP}. We start with an overview of related work on which we proceed to build
our method.

\section{Hierarchical Reinforcement Learning Methods}
While Hierarchical Learning is based on a features hierarchy, hierarchical \Gls{RL}
is based on a hierarchy of tasks. At the top of the structure, a manager designates
an employee with an assigned goal. The employee attempts to accomplish the goal to
get the reward before a termination condition is met (it gets fired). Employees can
be their own managers for other sub-employees, which is known as "Feudal
Reinforcement Learning" \citep{FeudalRL} with the "manager" (Lord) policy at the
top and "sub-managers" at lower levels of the hierarchy receiving sub-goals from
their "super-managers".

Some of the earlier work on \Gls{HRL} used the terms "primitive actions" and
"macro-actions". In \citep{MacroAction}, \citeauthor{MacroAction} introduced the
macro-actions as local policies that act in certain regions of the state space. The
"primitive action" policy then clusters the state space and assigns each
macro-action to a specific state space domain.

In this thesis, we will use the terms \textit{gate} or \textit{gating} policy for
the primitive actions, and "option" policies for the macro-actions.

\subsection{Hierarchical RL models}
An option is a tuple $(I_o,\pi_o, T_o)$, where $I_o \subset \cS$ is the initiation
set, $\pi_o$ is the option's policy, and $T_o \subset \cS$ is the termination set
(stopping time). The gating policy, in a similar way to the primitive policies from
the first chapter, is a mapping $\pi_g : \cS\times \mathcal{O} \rightarrow [0,1]$
from the state space to the distributions over the option set $\mathcal{O} = \{o_i
| i\in I \}$, the set of options. Following the same notations of the first
chapter, we denote the gating policy by $\pi_g$ and $\pi(.|s,o)$ the policy of the
option $o$.\par

We define the hierarchical policy as:

\begin{align}
	\pi(a|s) = \sum_{o \in \mathcal{O}} \pi_g(o|s) \pi(a|s,o).
	\label{eq:hpolicy}
\end{align}

Using the discounted state distribution $\rhopi$ as in equation (~\ref{eq:etarho}),
the objective becomes:

\begin{align}
	\eta(\pi) = \int_\cS\int_\cA \rhopi(s) \sum_{o \in \mathcal{O}} \pi_g(o|s) \pi(a|s,o)r(s,a)dsda
	\label{eq:hrlpolicy}
\end{align}

Several methods used this structure as in the \cite{HREPS} and \cite{TakayukiHRL},
where, besides the raw state, the gate provides the option with a latent variable
$o$ called "goal" or "sub-goal". In the case of a sub-goal, the options maximize
the expected accumulated intrinsic reward returned from an internal critic.

In a more general framework, the hierarchical policy picks an option at times $S_t
= \sum_{l=1}^{t} T_{o_l}$, where $o_l$ is the option drawn at round $l$. Between
$S_t$ and $S_{t+1}$, the actions are chosen following option $o_t$. We can write
the objective $\eta(\pi)$ as:
\[
	\eta(\pi) = \Epi{\sum_{t=1}^{\infty}\gamma^{S_{t-1}} \sum_{u=0}^{T_{t}}\gamma^u r(s_t,a_t)}.
\]

Ultimately, the HRL model should also learn the stopping time for each option.

\subsection{Related work}

In their hierarchical adaptation of the DQN \citep{HDQN}, \citeauthor{HDQN} used a
gate (meta-controller) and options (controllers) to attain handcrafted goals. While
the gate tries to maximize the accumulated reward from the environment (extrinsic
reward), the options maximize an intrinsic reward from the internal critic by
achieving the assigned goals. The abstract action terminates when the episode ends
or the goal is reached. Their method succeeded in solving the "MontezumaRevenge
Atari" game, in which the non-hierarchical methods fail to score. However, the hDQN
goals were handcrafted and hence needed input that can be expensive.

For an on-policy variant, \citeauthor{HREPS} \cite{HREPS} introduced several
Hierarchical REPS variants (see \autoref{sect:PGM} in \autoref{chap:chapter1}). The
option policies in \gls{HREPS} were allowed to share experiences (inter-option
learning), the chosen option was a hidden latent representation that the model
tries to infer using Expectation-Maximization. The inference is based on a sampled
trajectory, in which each policy $o_i$ attempts to "claim" its responsibility
$p(o_i|s,a)$ for action $a$ at state $s$. To avoid concurrence between policies,
\gls{HREPS} constrains the entropy of options' responsibilities.

Among the recent research on HRL, \citeauthor{offpolicyterminate} in
\cite{offpolicyterminate} addressed the learning of the option's stopping time
$T_o$. In the call-and-return HRL model, as in h-DQN, an option is run until
completion. In their paper, they argued for the efficiency of shorter option
policies and derived a method to learn the termination condition in an off-policy.
In our work, however, this aspect is not treated, and we restrict our study to
fixed option duration.

As pointed out in \cite{MacroAction}, the gating policy performs a clustering of
the state space domain. While in \gls{HREPS} \cite{HREPS}, option policies could
achieve the same performance on some domains of the state space. This indicates
that the choice of the optimal gating policy $\pi_g$, if achievable, isn't unique.
Since we treat $o$ as a latent variable, in analogy with \gls{HREPS}, we need to
impose additional constraints to obtain a preferable solution. For the choice of
the constraint, we opt for \Gls{RIMs}.

\section{Hierarchical Reinforcement Learning via Information Maximization}
To learn the hierarchical policy, we impose a regularization on the latent variable
$o$. The hierarchical policy $\pi$ then attempts to optimize the objective:
\begin{align}
	\mathcal{L}(\pi) = \eta(\pi) + \mathcal{R}(\pi),
\end{align}

where $\eta(\pi)$ was defined in \autoref{eq:hrlpolicy}.

We choose to impose a mutual information constraint on the gating policy. In the
\Gls{RIM} framework, $\eta(\pi)$ is considered the regularization term of $\pi$.

\subsection{Regularized Information Maximization }

In Unsupervised Learning, Mutual Information measures have been shown to produce
interpretable latent representations, especially in the clustering task
(\cite{HuIM}, \cite{LSIM}, \cite{SSLIM}). Following the previously discussed
models, learning the gating policy in our model is similar to performing a
clustering of the state space and assigning each option to a specific domain. It is
hence justified to leverage \gls{MI} performance in the HRL task.\par

The gating policy, a discrete distribution over $\mathcal{O}$, is the
representation $O$ and $X=(s_i,a_i)_i$ the observed variable. We define the
empirical mutual information:

\begin{align}
	\hat{I}(O,X) = \hat{H}(O) - \hat{H}(O|X),
\end{align}

where $\hat{H}(O)$ is the empirical estimate of entropy $H(O) = \int p(o)
\text{log}(p(o)) = \Ea{\text{log} (p(O))}$ using the empirical distribution:
\[
	\hat{p}(o) = \by{n} \sum_{i=1}^{n} p(o|s_i,a_i),
\]
and $\hat{H}(O|X)$ is the empirical estimate of $H(O|X)$
\[
	\hat{H}(O|X) = \by{n}\sum_{i=1}^{n} p(o_i|s_i,a_i)\text{log}(p(o_i|s_i,a_i)),
\]

with: $p(o|s,a) = \frac{\pi(a|s,o)\pi_g(o|s)}{p(a|s)}$, and $p(a|s) = \sum_{o \in
\mathcal{O}}^{ } \pi_g(o|s)\pi(a|s,o) $

Adding the regularizing quantity $\Hat{\mathcal{R}}(\pi)$ to the objective
$\eta_\pi$ incentivizes the gating policy to uniformly sample option policies by
increasing $H(O)$, while at the same time, clearly distinguishing between each
option's domain by decreasing $H(O,X)$.

\subsection{Objective function}

We fix the number of options policies $|O| = k$ and we consider a deterministic
termination set $\mathcal{T_o} = \tau$. We consider the hierarchical policy as in
(\ref{eq:hpolicy}). The hierarchical policy's objective :

\begin{align}
	\eta(\pi) = \int_\cS\int_\cA \rhopi(s) \sum_{o \in \mathcal{O}} \pi_g(o|s) \pi_o(a|s,o)r(s,a)dsda - \lambda I(O,S),
\end{align}

for a given parameter $\lambda$.\par

Since we will be using parametric approximation, we note $\pith$ the policy with
parameter $\theta$. In a similar way to \gls{TRPO} notations, we use $\theta$
indices to refer to quantities defined using $\pith$.

\subsection{On-Policy HRL:}

The TRPO model incorporates an estimation of the advantage function $\Api$, which
can be used to model the gating policy using the softmax probabilities. Then the
probability of choosing the option $i$:
\begin{align}
	\pi(o|s,a) = \frac{exp[A(s,\pi_i(s)|o_i)]}{\sum_j exp[A(s,\pi_j(s)|o_j)]}
\end{align}

This would allow to decouple the learning of the latent variable o and the gating
policy, since the distribution on options only depend on their estimation
advantages. From the \gls{HREPS} and \gls{h-DQN} methods introduced, we can either
\begin{itemize}
	\item In a similar way to \gls{REPS}, allow inter-option learning
	\item Make each policies value estimation hidden from the gating policy. This
	      implies that the gating policy should estimate its own value function as
	      well. This similar to the \gls{h-DQN} method.
\end{itemize}

In this work, we present a different method where each option learns only from its
own experience while the gating policy is oblivious to the advantages estimation
(we won't use softmax Q).
\subsection{Hierarchical Kullback-Leibler bound}

To extend the TRPO to hierarchical policies, we use the following lemma (see proof
\ref{thm:proofKL} in Appendix).
\begin{lemma}

	Let $n\geq 1$ be the number of option policies, and $\pi_g$ the gating policy
	(distribution over the options set). If $\pi$ and $\tilde{\pi}$ are two
	hierarchical policies such that $\tilde{\pi}$ is absolutely continuous with
	respect to $\pi$, then at any state $s$, we have:
	\begin{align}
		\label{eq:hrlkl}
		\text{D}_{\text{KL}}(\pi(.|s)|\tilpi(.|s) \leq \text{D}_{\text{KL}}(\pi_g(.|s),\tilpi_g(.|s)) + \sum_{o\in \mathcal{O}} \pi_g(o|s)\text{D}_{\text{KL}}(\pi(.|s,o)|\tilpi(.|s,o))).
	\end{align}
\end{lemma}

The right side's second term is, in fact, the Bregman divergence for the
conditional entropy but with an expectation over the options set instead of the
state space as seen in Chapter 1.

In a similar work \cite{HRLParameters}, \citeauthor{HRLParameters} used a hierarchy
of parameters. Instead of drawing an option, the gating policy chooses the subspace
for the option's parameters. The options then belong to the family of multivariate
Gaussians. This choice of structure turns the inequality in \ref{eq:hrlkl} into an
equality, as the divergence between two hierarchies over the same option vanishes.

Based on \ref{eq:hrlkl}, to bound the KL divergence, it's sufficient to bound both
the gate and the options. We can also use adaptive bounds for the options by
bounding $\pi_g(o|s)\text{D}_{\text{KL}}(\pi(.|s,o)|\tilpi(.|s,o)))$ instead of
$\text{D}_{\text{KL}}(\pi(.|s,o)|\tilpi(.|s,o)))$.

\subsection{Algorithm and Experiments}

\begin{algorithm}
	\caption{Trust Region Hierarchical Policy Optimization}
	\begin{algorithmic}
		\Require{A parameterized gate policy $\pi_\theta$ and option policies $(\pi_{o_i,\theta})$ }
		\For{$i \gets 1 \textrm{ to } T $}
		\State{Run policy for N trajectories}
		\State{Estimate the advantage}
		\State{Calculate the Mutual Information Measure}
		\State{Use Trust Region method to get the update step wrt to the gate}
		\State{Update Step}
		\For{ each option i}
		\State{Get the Objective after last update}
		\State{Use the trust region method for the option i}
		\State{Apply update}
		\EndFor
		\EndFor
	\end{algorithmic}
\end{algorithm}

The trust region method mentioned in the algorithm is the same one used in standard
TRPO, by using a quadratic approximation of the KL divergence, then exploiting the
Fisher vector product and using conjugate gradient. The objective needs to be
calculated after each update and also incorporate gradient graphs for each policy
when using neural networks. In practical terms, it requires large GPU memory when
naively implemented. A trickier way would imply stopping the gradient from policies
uninvolved in the optimization step.

We compare the performance of the TRHPO method with the TRPO in the 4 rooms
environment. We use the same neural network architectures for the gating and option
policies and the TRPO policy with the same KL bounds.

\singlefig{0.5}{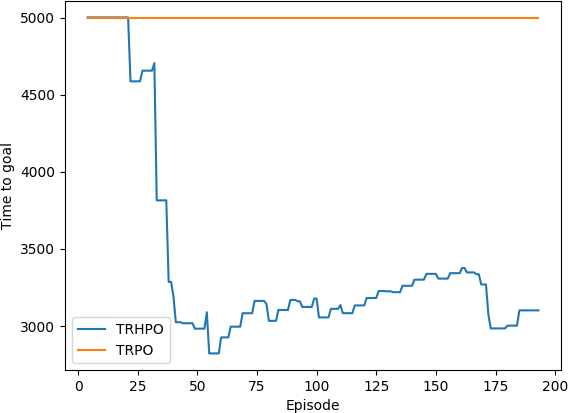}{TRHPO vs TRPO in 4 Rooms}{fig:trhpo1}

While the TRPO performance is far from optimal, it outperforms the TRPO method in
MDPs with sparse rewards, a known advantage of hierarchical methods.\par
We also plot the exploited options during the learning process:

\singlefig{0.5}{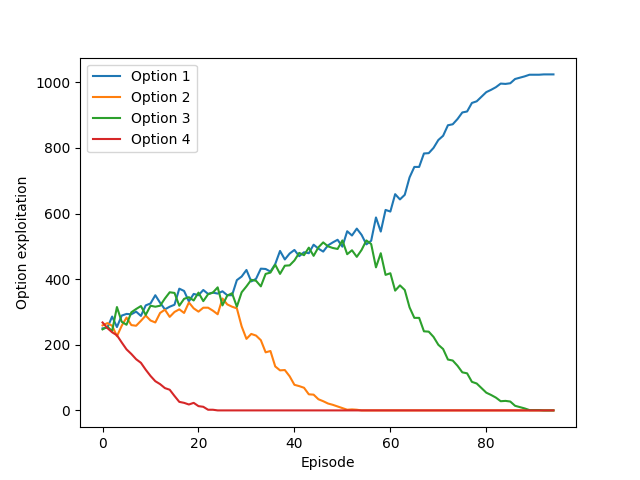}{TRHPO exploited options}{fig:trhpo2}

\section*{Conclusion}

Despite outperforming simple TRPO, the TRHPO gradually excludes option policies to
only exploit one option policy around the 90th episode. At the end of the first 100
episodes, the algorithm reaches a local optimum. In our tests, convergence to local
optima limits the TRHPO performance in a range of environments, especially where
the reward isn't sparse.\par

After building the hierarchical policy learning method, we aim next to incorporate
options that can be learned independently of the top-level policies. We aim to
achieve such a goal by generalizing the \Gls{PVF} concept to solve the MDP on
specific basis functions.

\chapter{Spectral Framework for options discovery}

\begin{abstract}
	After working on temporal action abstractions, we tackle the problem of state
	space abstraction to learn options on basis functions induced from the graph of
	the environment. We use spectral basis functions to learn representations that
	reflect the geometric dynamics of the environment.
\end{abstract}

\section*{Introduction}

Another way of learning the value functions or policies is to approximate them on a
certain basis functions. This approach has several applications in Machine
Learning\cite{MallatWavelets} (Fourier Wavelets, Laplacian) and Quantum Physics
(Hamiltonian)\cite{JoffreQuant} and yields good results when the right basis is
used. The most recent one perhaps is the $100\%$ accuracy on MNIST using complex
neural networks built on the frequency domain features\cite{AizerbergMLMVN}\par
\glsreset{PVF}
In the context of \gls{RL}, the concept of \Gls{PVF} introduced by
Mahadevan\cite{MahadevanPVF} uses the eigenvectors of the graph Laplacian as basis
to learn the optimal policy for a discrete state space \gls{MDP}. The idea is to
exploit the geometry of the environment since spatially close states can be
"dynamically" far or even separate. In the 4-room environment in the examples,
$state_1$ is visually closer to $state_2$ than $state_3$, but dynamically further
than it.

\triplefig{
	\subfig{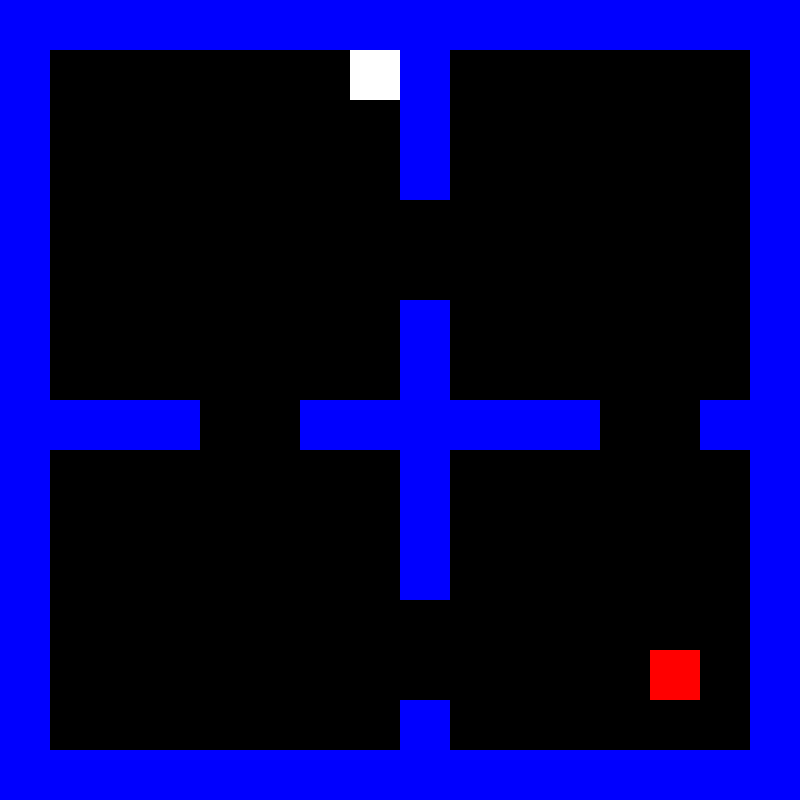}{0.7}{$state_1$}{fig:s1}}{
	\subfig{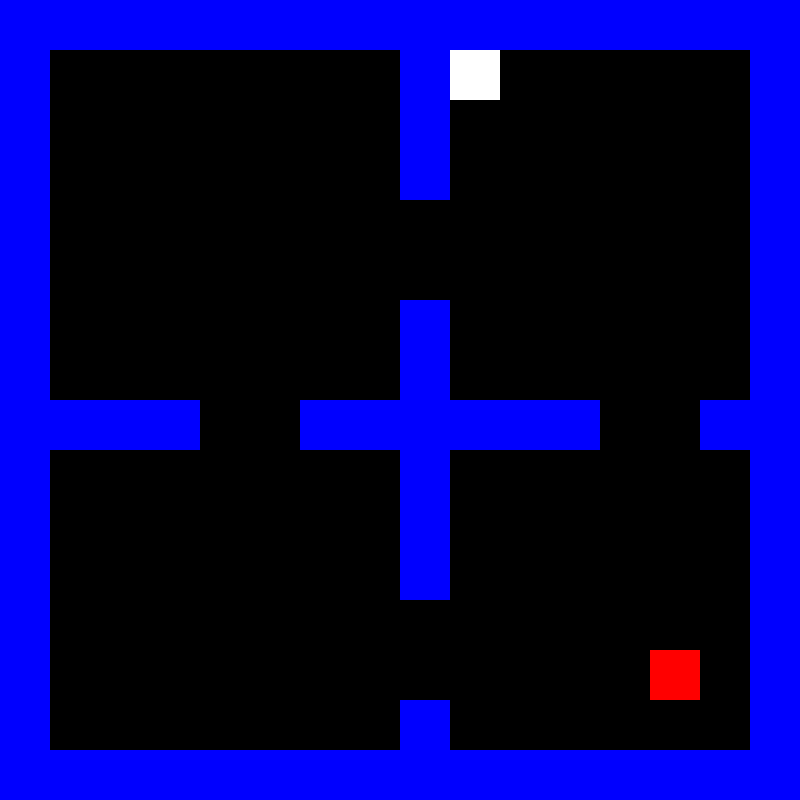}{0.7}{$state_2$}{fig:s2}}{
	\subfig{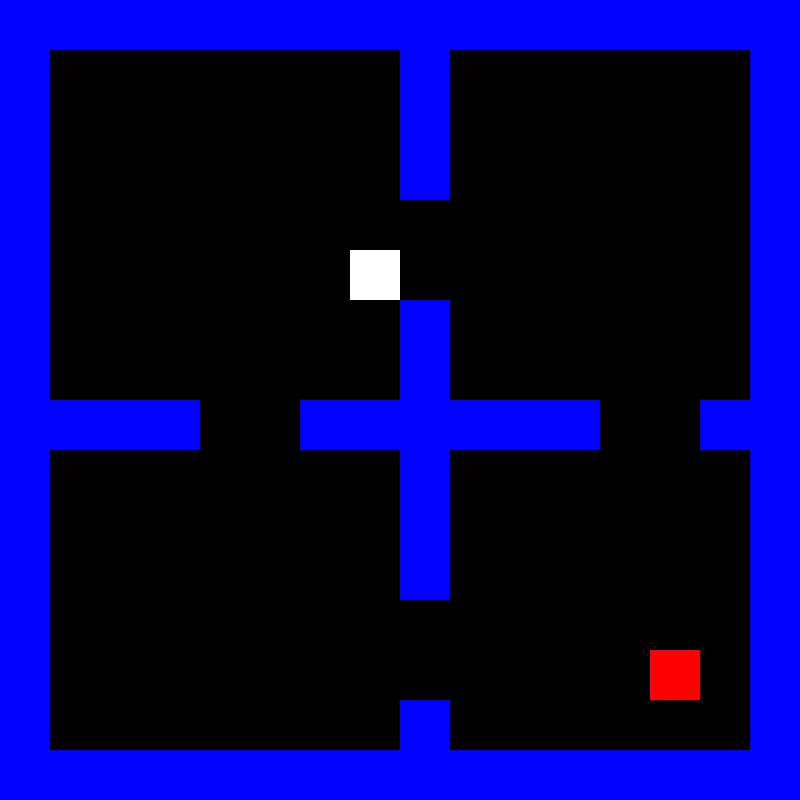}{0.7}{$state_3$}{fig:s3}}{
	Euclidian distance and geodesic distance}{
	fig:DynamicState
}

\glsreset{PVF}
In this chapter, we will introduce the \Gls{PVF} paradigm, present the Laplacian
operator and its use in Spectral Clustering before integrating it into our study of
\gls{RL}. Since we are interested in applications for high-dimensional state space,
we will present and evaluate our model for scaling the introduced techniques.

\section{Proto-Value Functions and Spectral Clustering}
The \Gls{PVF}~\cite{MahadevanPVF} introduced a spectral framework for approximate
dynamic programming. It mainly addressed a finite state space MDP setting in
model-free \gls{RL}. The overall framework can be summarized as follows: the agent
constructs the adjacency matrix of the state space graph where the similarity
$W_{i,j} = 1$ if the agent could move from state $s_i$ to state $s_j$. The
diffusion model is defined using the normalized symmetric graph Laplacian
$L_{sym}$. Then the basis functions are simply the smoothest eigenvectors of the
Laplacian (associated with the lowest eigenvalues). To learn the optimal policy,
PVF uses "Representation Policy Iteration," a policy iteration that takes the
projection on the smoothest eigenvectors as states.\par

To illustrate the concept, we use a 4-room environment with a size of $16\times16$
(177 accessible states). We consider the fixed starting point with an absorbing
goal state. We have 176 states, for which we plot in (~\ref{fig:OptimalApprox}) the
optimal value function $V*$ and its approximation on the first 5 and the full
Laplacian basis, learned using \gls{LSPI}.
Contrary to what we might conclude visually from (\ref{fig:vfeig5}), using the
smoothest 5 eigenvectors doesn't solve the problem. In fact, the approximated
policy has absorbing states different from the target state. Solving this issue
requires using more eigenvectors which in turn creates new modes in the
approximated function, hence the need for a sufficiently large basis to nullify the
discontinuity. Several works address this problem by using a smoother basis as in
\citeauthor{Sugiyama2008} where the geodesic Gaussian kernels are used to build
similarities based on the shortest path distance.

In an entirely different perspective, with the concept of "Eigenoptions",
\citeauthor{Machadoeig} learns options that maximize intrinsic rewards induced by
the eigenvectors. This allows creating options that reflect the geometric
properties translated by each eigenvector separately. Afterwards, a policy over the
Eigenoptions is used to associate each eigenoption with a certain domain of the
state space, similar to hierarchical RL.

\triplefig{
	\subfig{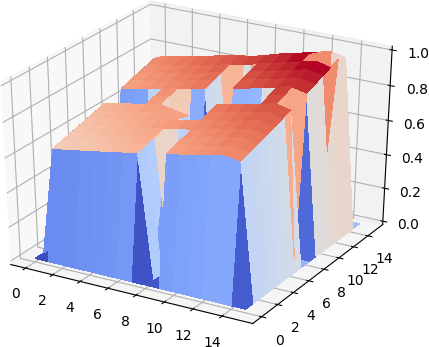}{0.9}{Optimal $V^*$}{fig:VFunc}}
{\subfig{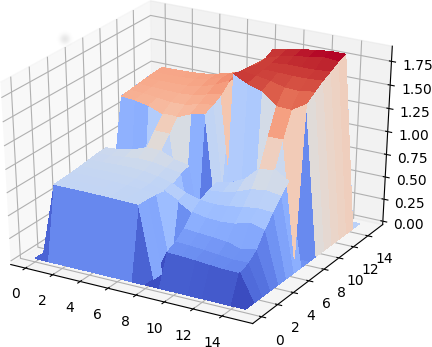}{0.9}{ Using 5 eigenvectors}{fig:vfeig5}}
{ \subfig{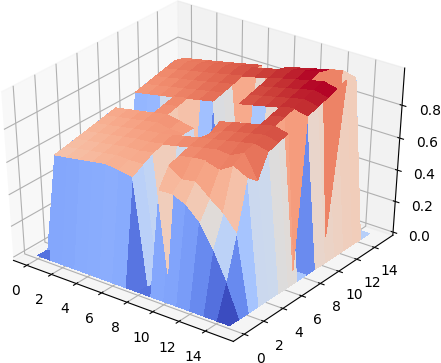}{0.9}{using 177 eigenvectors}{fig:vfeig177}}{
	Optimal $V$ and approximations}{fig:OptimalApprox}

In the subsequent section, we delve into fundamental properties of the graph
Laplacian that will be utilized subsequently to introduce eigenoptions.

\subsection{Laplacian Operator and Spectral Clustering}

Spectral clustering\cite{SpectralClustering} addresses the problem of partitioning
an undirected weighted graph $G=(V,E)$ under specific constraints, where $V$
represents the set of vertices and $E$ the set of edges.\par

Let $V =\{{v_1,...,v_n}\}$, and $W$ the $n \times n$ adjacency matrix where
$W_{i,j}$ is the non-negative weight of the edge linking $v_i$ and $v_j$, and
$W_{i,j}=0$ means that the vertices $v_i$ and $v_j$ are not linked. The undirected
nature of the $G$ means that $W$ is symmetric. We define the degree of the vertex
$v_i$ as $d_i = \sum_j W_{i,j}$ and the degree matrix $D=\diag{(d_1,....,d_n)}$.\par

Let $A \subset V$, we define the indicator vector $1_A=(f_1,.....f_n)$ as: $f_i=
1_{v_i \in A}$. Similarly, we can measure the similarity between two subsets $A$
and $B$ by the sum of the "similarity" between their vertices: $$ W(A,B) :=
\sum_{i\in A,j\in B} w_{i,j} = 1_A^TW1_B \text{ ,} $$ and to measure the size of
the set $A$, we can either use:
\begin{itemize}
	\item The cardinality measure : $$\text{vol}_\text{c}(A) \vcentcolon= |A| =
	      1_A^T 1_A$$.
	\item Total of its elements degrees: $$\text{vol}_\text{d}(A) \vcentcolon=
	      \sum_{i\in A} d_i = 1_A^T D 1_A $$.
\end{itemize}

\subsubsection{Graph cut and Laplacians}

To partition the graph $G$ into two separate sets $A$ and $\overline{A}$, Spectral
Clustering searches for the least "similar" sets $A$ and $B$. We can use two types
of similarity measures:
\begin{align}
	\text{Ratiocut}(A,\overline{A})\vcentcolon=\frac{W(A,\overline{A})}{|A|} \\
	\text{Ncut}(A,\overline{A})\vcentcolon=\frac{W(A,\overline{A})}{\text{vol}_\text{d}(A)}
\end{align}

We can prove the following properties:
\begin{align}\label{eq:rayleighquo1}
	\text{Ratiocut}(A,\overline{A}) & =\frac{ {\indicator_A}^T (D-W) \indicator_A }{{\indicator_A}^T\indicator_A}     \\\label{eq:rayleighquo2}
	\text{Ncut}(A,\overline{A})     & =\frac{{\indicator_A}^T (D-W) {\indicator_A} }{{\indicator_A} ^T D\indicator_A}
\end{align}

Since we are interested in finding the set $A$, distinct from $V$, the partitioning
problem is then formulated as finding the vector of $\Real^n$ with $\{0,1\}$
elements that minimizes the Rayleigh quotient in Equation \autoref{eq:rayleighquo1}
or \autoref{eq:rayleighquo2}.

\subsubsection{Graph Laplacians}

\paragraph{Unnormalized/Combinatorial Laplacian:} The combinatorial Laplacian, appearing in the Ratiocut, is defined as
$L\vcentcolon=D-W$, which has the following properties:
\begin{itemize}
	\item for $f$ in $\Real^n$ $$f^T L f = \sum w_{i,j}(f_i-f_j)^2$$
	\item L is symmetric semi definite
	\item The smallest eigenvector is $1$ for the eigen value $\lambda_0=0$
	\item In particular, self-edges in a graph do not change the corresponding
	      graph Laplacian
\end{itemize}
\paragraph{Normalized Laplacians:}
Spectral Clustering literature uses two types of normalized Laplacians:
\begin{itemize}
	\item Random Walk $L_{rw} = I - D^{-1}W$
	\item Symmetric, appearing in the Ncut, $L_{sym} = I - D^{-\by{2}}WD^{-\by{2}}
	      $
\end{itemize}

The matrix $T=D^{-1}W$ can be viewed as a transition matrix. Although $L_{rw}$ is
not symmetric, all its eigenvalues are real, non-negative, and upper-bounded by
$1$. The normalized Laplacians exhibit the following properties:
\begin{itemize}
	\item For $f\in \Real^n$:
	      \begin{align}
		      f^T L_{sym}f = \by{2}\sum_{i,j}(\frac{f_i}{\sqrt{d_i}}-\frac{f_j}{\sqrt{d_j}})^2
	      \end{align}
	\item $u$ is an eigenvector of $L_rw$ with eigen value $\lambda$ iif $W=D^{\by{2}}u$ is an eigenvector of $L_sym$ for the same eigen value.
\end{itemize}

\subsubsection{Cuts optimization}

Solving \autoref{eq:rayleighquo1} and \autoref{eq:rayleighquo2} on the set
$\{0,1\}^n$ is an NP-hard problem. To address this challenge, we relax the
constraint such that the unknown function $f$ is dense with a specific norm.
Minimizing the Ratiocut involves identifying the minimizer $f$ of the Rayleigh
quotient for the matrix $L$, subject to the conditions $f^T f = \sqrt{n}I$ and $f^T
\indicator_V = 0$.

For the Ncut, the objective is to determine $g=D^{\by{2}} \indicator_A$, which is
the minimizer of the Rayleigh quotient for $L_{sym}$. This is achieved under the
constraints $g^T g = \sqrt{n}I$ and $g^TD^{\by{2}}\indicator_V = 0$. The minimizer
of the Rayleigh quotient, subject to these specified constraints, corresponds to
the smoothest eigenvector of the matrices (excluding the first eigenvector).

Consequently, vertices within the same subset are anticipated to exhibit similar
projections onto the solutions of the relaxed problem. As we ascend the Laplacian
spectrum, the projections onto the eigenvectors are inclined to capture more
distinct (higher frequency) variations. Therefore, possessing knowledge of the
complete spectrum is expected to yield representations that effectively
differentiate all vertices.

\section{Eigenoption Discovery}

Eigenoptions\cite{Machadoeig} are based on a simple property of the eigenvectors.
Plotting the first 3 eigenvectors below, the agent starts at position $(3,3)$ and
has to reach $(33,33)$. Let $\phi_i$ be the $i-th$ eigenvector, and we define the
intrinsic reward $r_i$ for taking action $a$ to transition $s\rightarrow '$ $$
r_i(s,a) = \phi_i(s')-\phi_i(s) $$ The $i$-th Eigenoption is the policy that
maximizes the expected accumulated intrinsic reward $r_i$. In
\autoref{fig:first3eig}, the first Eigenoption would lead the agent to the corner
of the first room an and exit the third room. For the third Eigenoption, the agent
would have to exit the second and fourth room to reach the absorbing state in the
corners of the first and third room. So if the task is to exit the first room, we
should follow the second eigenoption.

\triplefig{ \subfig{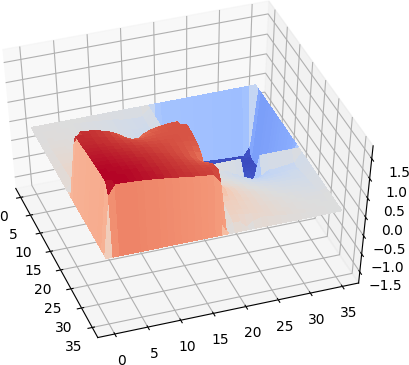}{0.9}{$1^{st}$
		eigenvector}{fig:3eig1}}{ \subfig{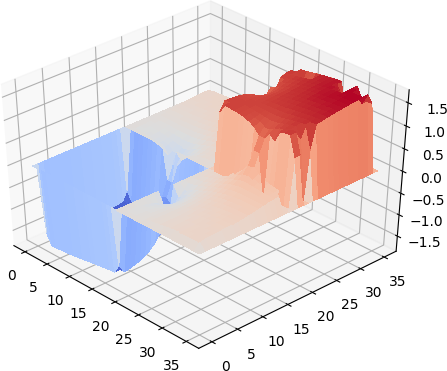}{0.9}{$2^{nd}$
		eigenvector}{fig:3eig2}}{ \subfig{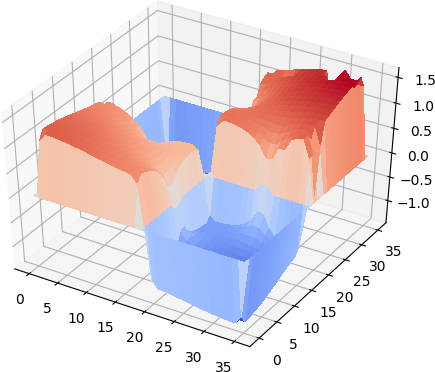}{0.9}{$3^{rd}$
		eigenvector}{fig:3eig3}}{ Eigenvectors of $L_{sym}$}{fig:first3eig}

Extending the concept of Eigenoptions to large and continuous state space \gls{MDP}
requires to be able to learn eigen vectors of large, and virtually infinite, graph
Laplacian. In this section we go through two recent methods and in the last section
we present our own.

\subsection{Optimization on Matrix Manifold}
For problems involving orthogonality or low-rank constraints, these characteristics
are expressed naturally using the Grassmann manifold. There has been growing
interest in studying the Grassmann manifold and exploiting its structure,
especially in computer vision. However, the manifold structure was usually used to
model either the input or the parameters of parametric functions. It is only
recently that we've started to see works where neural networks are used to output
representations on the manifold. More specifically, given a batch sample
$(X_i)_{i\leq n}$ in $\Real^d$, we want to learn a mapping $F: \Real^d \rightarrow
\Real^k$ that solves:
\begin{align}
	\label{eq:ObjectSpin}
	\begin{aligned}
		 & \underset{F}{\text{minimize}} &  & \mathcal{L}((X_i)_i,F) = \sum_{i,j\leq n} K(X_i,X_j)F(X_i)^T F(X_j) \\
		 & \text{subject to}             &  & \by{n}\sum_i F(X_i)^T F(X_i) = I_k,
	\end{aligned}
\end{align}

with K a similarity kernel on $\Real^d\times\Real^d$.\par

In other words, the $n\times k$ matrix $Y$ with rows $Y_i = \by{\sqrt{n}}F(X_i)$
should be on the Stiefel manifold defined as: $$ V_{n,k} \vcentcolon= \{A\in
\Real^n\times \Real^k | A^T A = I_k\} $$

However, the problem doesn't have a unique solution. Given an optimal mapping F and
taking a matrix $ Q\in \mathcal{O}_k $ to define $ {F'}(X) = F(X)^T Q $, then
${F'}$ is also optimal.\par

To correctly define the problem, the optimization set should identify elements of
$V_{n,k}$ whose columns span the same subspace: the Grassmann manifold $G_{n,k}$.

When $n$ is large, evaluating whether $F$'s outputs are on the Grassmann manifold
can be only approximated. Informally, if the input X has distribution P, then $F$
is simply transforming X into a latent representation such that:
\begin{align}
	\label{eq:ObjectSpinProba}
	\begin{aligned}
		 & \underset{F}{\text{minimize}}
		 &                               & \E^P[K(X,X')F(X)^T F(X')] \\
		 & \text{subject to}
		 &                               & \E^P[F(X)^T F(X)] = I_k,
	\end{aligned}
\end{align}

To evaluate if the output of batch of size $n$ is on the manifold, we use the
distance: $$ d_g(A) = \underset{B \in G_{n,p}}{inf}{ ||A-B||_2^2} $$

\begin{lemma}
	For a matrix $A \in\Real^n\times \Real^k$:
	\begin{align}
		d_g(A) = \sum_{i=1}^{k} (\sigma_i-1)^2
	\end{align}
	where $S = \diag(\sigma_1,...\sigma_k)$ such that $A=USV^T$ is the SVD decomposition of A.
\end{lemma}
Consequently, the projection of $A$ on Grassmann manifold is $B=UV^T$.

\subsection{Neural Network for Eigenvectors learning}

\subsubsection{Spectral Net}
In \cite{spectralnet}, \citeauthor{spectralnet} used neural networks to learn an
approximation of Laplacian eigenvectors for large-scale applications. Their
'SpectralNet' $f_\theta : \Real^n \rightarrow \Real_k$ tries to minimize the
"spectral loss" defined as: $$ L_{SpectralNet}(\theta,(x_i)_i) =
\by{m^2}\sum_{i,j=1}^{m}W_{i,j}\normtwo{f_\theta(x_i)-f_\theta(x_j)}^2, $$ under
the constraint: $$ \by{m}\sum_{i,j=1}f_\theta(x_i)^Tf_\theta(x_j) = I_k. $$ The
minimization is done with batches, and $m$ is the size of the batch. To impose the
orthogonality constraint, a "Cholesky" decomposition layer is added on top of the
neural network.\par
\paragraph{Cholesky Layer:}
Let $Y$ be the neural network's output for a sample $X$ of size $m$. If we take the
Cholesky decomposition $L$ of $S:=\by{m}Y^T Y = L L^T$, where $L$ is a lower
triangular matrix, then $Y^* = \sqrt{m} Y L^{-1}$ is on the Grassmann manifold.\par

The paper's main contribution, however, is establishing a theoretical result for
the minimum number of units required to be able to learn the eigenvectors.

\begin{lemma}
	\begin{align}
		VC dim(F^{spectral clustering}_n)\geq \frac{n}{10}.
	\end{align}
\end{lemma}

This implies that to learn representation separating $n$ points, the number of
weights in the neural net should be in the same order as the number of points $n$.

They have also demonstrated the convergence to real eigenvectors for some cases. On
the other hand, the methods used for learning are based on simple back-propagation.
At each iteration, the Cholesky layer weights are frozen to minimize the Spectral
loss. Since the minimizer of such loss is the zero matrix, the output approaches
zero while the Cholesky layer weight, the inverse of $L$ for $Y$ approaching $0$,
can explode if the learning rate isn't carefully controlled. Furthermore, the
SpectralNet methods don't address the overlapping of eigenvectors; hence, the
representation learned is, in fact, a noisy combination of the smoothest
eigenvector.

\subsubsection{Spectral Inference Networks}

Inspired by applications for the Hamiltonian in quantum mechanics,
\citeauthor{Spin} addresses the general problem of maximizing the Rayleigh quotient
for an "infinite" matrix whose elements are defined by a certain similarity kernel.
\begin{align}
	\begin{aligned}
		 & \underset{x}{\text{maximize}} &  & \text{Tr}((Y^T Y)^{-1} Y^T A Y) \\
		 & \text{subject to}             &  & Y^T Y = I_k,
	\end{aligned}
\end{align}

where $A \in \Real^{n,n}$ and $Y \in \Real^{n,k}$.

Given an input data $(x_i)_{i\leq n}$ and $(y_i)_{i\leq n}$, their output using
\gls{SpIn}, and $A$ as their similarity matrix, we define $\Pi$ and $\Sigma$ in
$\Real^{k,k}$ as:

\begin{align}
	\Pi    & \vcentcolon= \by{n}\sum_{i,j} a_{i,j} x_i{x_j}^T \\
	\Sigma & \vcentcolon= \by{n}\sum_{i,i} x_i{x_i}^T
\end{align}

Using these notations, $\Sigma$ and $\Pi$ depend on the network parameter $\theta$,
and the gradient of the objective can be written as:
\[
	\text{Tr}(\Sigma \nabla_\theta\Pi) - \text{Tr}(\Sigma^{-1}\Pi\Sigma^{-1}\nabla_\theta\Sigma)
\]

To solve the objective, \citeauthor{Spin} proceeds by accumulating an empirical
estimate of $\Sigma$ and $\nabla_\theta\Sigma$, which requires calculating $k^2$
gradients and storing $k^2$ times the number of network parameters. Based on the
theoretical result of SpectralNet, approximating eigenvectors requires a large
neural network. Therefore, calculating and storing $\nabla_\theta\Sigma$ is
expensive.

On the other hand, \gls{SpIn} addresses the question of overlapping eigenvectors
and modifies the gradient in a way to make the update sequentially independent: the
first eigenvectors' gradient is independent of the gradient coming from
eigenvectors of higher eigenvalues. In our implementation of the \gls{SpIn},
modifying the gradient reduces the update step and hence renders the objective
increments extremely slow. Furthermore, we believe that the sparsity constraints
used in their neural network and the separation of weights of eigenvectors have an
important impact on the performance demonstrated in their paper, which was
considered a technical detail in their work.

\subsection{Proposed Spectral Network}

We build on SpectralNet and \gls{SpIn} to derive a new method that produces
separated eigenvectors, with fewer gradient backpropagations and memory
requirements. Using the same notation in the previous subsection, we aim to
minimize the "Sequential" Rayleigh quotient. Let $$ R_j(A,Y) = \text{Tr}((Y_j^T
Y_j)^{-1} Y_j^T A_{:j,:j} Y_j) $$ Where $Y_j$ is the matrix of the first $j$
vectors of $Y$ and $A_{:j,:j}$ is the first $j\times j$ block of matrix $A$.

Then our objective is:

\begin{align}
	\label{eq:SpInobj}
	\begin{aligned}
		 & \underset{Y}{\text{minimizer}} &  & \sum_j R_j(A,Y) \\
		 & \text{subject to}              &  & Y^T Y = I_k,
	\end{aligned}
\end{align}

If the eigenvalues of $A$ are distinct, then the columns of the global optimum are
guaranteed to be orthogonal.

From this point, the key idea is to take update directions for which the gradient
$\nabla_\theta\Sigma$ is small and update the network parameters to minimize the
Grassmann distance. This is similar in a way to the sub-gradient method. However,
the set on which we incrementally project (Grassmann Manifold) is not convex in our
case.

Furthermore, we use large samples to estimate the sequential Rayleigh quotient.
Alternatively, we can use the learned covariance matrix as in the SpIn.

At iteration $t$, where the parameter of the network is $\theta_t$:
\begin{align}
	\begin{aligned}
		 & \underset{\theta}{\text{minimize}}
		 &                                    & \sum_j R_j(A,Y_\theta)                                             \\
		 & \text{subject to}
		 &                                    & ||Y_\theta^T Y_\theta- Y_{\theta_t}^T Y_{\theta_t}||_2^2 = \delta.
	\end{aligned}
\end{align}

We note that $||Y_\theta^T Y_\theta- Y_{\theta_t}^T Y_{\theta_t}||_2^2$ is
approximately quadratic in $\theta$ for $\theta$ sufficiently close to $\theta_t$.
For a small $\delta$, the update step ensures staying approximately on the same
Stiefel manifold as the covariance doesn't change. After each step minimizing the
objective, we take an orthogonal update towards the Grassmann manifold under the
constraint of keeping the gained improvement. The learned covariance matrix is used
to update the weight of the output Cholesky layer. The update is done gradually
using a learning step instead of the brutal update as in SpectralNet. Ultimately,
the Cholesky layer allows the network to output the centered eigenvectors based on
the learned covariance matrix $\Sigma$.

\begin{algorithm}
	\caption{Spectral Network}
	\begin{algorithmic}
		\Require{A parameterized network $F_\theta$, Learning rate $\alpha$ for $\Sigma$, Learning rate $\beta< 1$ for gradient updates }
		\For{$i \gets 1 \textrm{ to } T $}
		\State{Sample $(x_i)_i$ and get the output $(y_i)_i$ }
		\State{ Calculate $\hat{\Sigma}$ and update the Cholesky layer weights with rate $\alpha$}
		\State{ Get gradient of the Sequential Rayleigh quotient}
		\State{ Use quadratic approximation of the constraint, calculate fisher information matrix}
		\State{ Use Conjugate Gradient to find a search direction $g$}
		\State{ Get the Grassmann distance gradient $p$ and project it on the normal on $g$ }
		\State{ Use Armijo line search using $g$}
		\State{ Use a simple line search using $p$ while keeping approximately   the gained improvement from previous step}
		\State{If Rayleigh Quotient doesn't improve after 10 iterations, reduce update rate}
		\EndFor
	\end{algorithmic}
\end{algorithm}
\newpage

\section{Evaluating the Spectral Network}

We use our network to learn eigenvectors of the combinatorial Laplacian for the 4
rooms environment of size 36. Since the eigenvector associated with the first
eigenvalue is the vector $1$, we can fix the first component of the output.

\subsection{Eigenvectors of the 4-rooms environment}
Two states are similar if the agent transitions from one to another. We consider
the grid with size $36$ with random locations for the agent and the target point,
which sums up to $1202312$ possible states. We use the combinatorial Laplacian in
our experiment, since all positions have nearly the same degrees, except positions
near the walls. Furthermore, using normalized Laplacians requires estimating the
degree of vertices, which is prohibitively difficult in batch learning. To avoid
sampling the same state multiple times, we use a graph builder that hashes the
states and identifies states using the hash dictionary.

We visualize the finite set of states for the deterministic 4-room environment. The
eigenvectors, however, have been learned for the stochastic 4-room environment,
which has a space of over $10^6$ states. We display in the figures below the 3rd
and the 5th learned eigenvectors for the deterministic environment and a second
environment which is a 180 degrees rotation of the initial one.

The obtained functions display similar invariance characteristics with spectral
eigenvectors. More importantly, learning eigenoptions using these "eigenfunctions"
will ensure that the learned options will lead to the same relative states. In our
work, we attempted to use the TRPO model used in the first 2 chapters to learn
eigenoptions, but without success. We believe it is due to the necessary large size
of the network.

\doublefig{\subfig{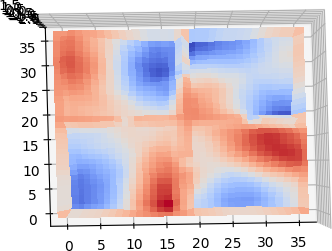}{0.7}{Initial Environment}{fig:331}}{\subfig{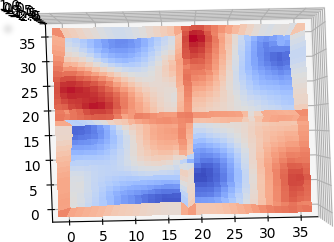}{0.7}{Rotated Environment}{fig:332}}{3rd Eigenvector}{fig:room3}
\doublefig{\subfig{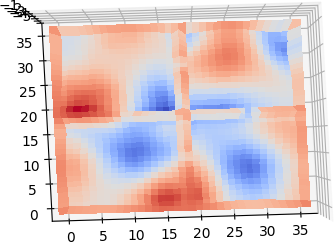}{0.7}{Initial Environment}{fig:351}}{\subfig{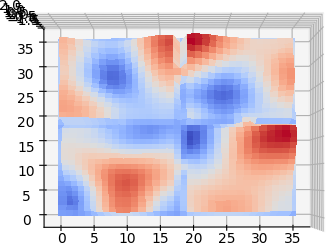}{0.7}{Rotated Environment}{fig:352}}{5th Eigenvector}{fig:room5}

\subsection{Clustering}

We also compare our method's performance with the baseline in SpectralNet, which
uses the spectral network without learning Siamese similarities or using the
variational embedding of the dataset.

We define similarities between data points using nearest neighbors for $n=5$. To
reduce the similarity variance between samples, we store the learned distances and
keep the closest $1024$ neighbors from each point. This allows us to stabilize the
similarities between data points once enough points are browsed. Using this method,
it is unclear what number of eigenvectors to use, since there is no guarantee of
having a fully connected graph or avoiding creating unnecessary cliques.

\paragraph{MNIST}
For MNIST, we report have the following performance:
\begin{center}
	\begin{tabular}{||c c c||}
		\hline
		Number of eigenvectors & Accuracy & SpectralNet Score \\ [0.5ex]
		\hline\hline
		10                     & 69.27\%  & 62.3\%            \\
		\hline
		11                     & 71.08\%  & -                 \\
		\hline
		15                     & 74.17\%  & -                 \\
		\hline
		20                     & 67.85\%  & -                 \\
		\hline
	\end{tabular}
\end{center}
\paragraph{Reuters}

For the Reuters dataset, we learn the first 4 eigenvectors. We train on $6.10^5$
data points and evaluate on a testing set of size $10^4$. We report a test accuracy
of $76.18\%$ after 72 iterations, which drops to around $69\%$ afterwards. The
training accuracy stabilizes around $67.24\%$. Since the training set is large, the
reported training accuracy is evaluated between distant iterations. This is to be
compared with the $64.5\%$ accuracy achieved by SpectralNet. Hence, we believe that
using embedding or Siamese networks, as was done by \citeauthor{spectralnet}, would
outperform SpectralNet.

\vspace{-1em}
\section*{Conclusion}

\vspace{-1em}
After building the HRL model, we attempted to learn options directly without going
through the full HRL method. We used the \gls{PVF} concept and built on recent work
to approximate eigenvectors using neural networks for large datasets. We
demonstrated that the learned functions present similar invariance characteristics
as the Spectral eigenvectors and showed that our method outperforms the recent
SpectralNet in the clustering task when used without Siamese distances nor the
variational embedding.\par

On the other hand, we encountered difficulties when learning eigenoptions using the
TRPO method. Therefore, we envision continuing to improve and build on the current
achievements and to extend it to the standard use of eigenfunctions: Successor
Features estimation.

\newpage
\chapter*{Conclusion}

Starting with the goal of building improved Reinforcement Learning methods, we
showed that clustering the state space is an efficient strategy to master complex
and large state spaces. First, using action abstraction by building the
Hierarchical TRPO method. In the second part, we build on the Eigenoption method,
which can be seen as a state abstraction, to learn the embedding of the state space
incorporating the knowledge about environment dynamics.\par

We perform extensive tests for our method for learning the eigenvectors for
clustering and for Eigenoptions estimation. Our next step was supposed to combine
the Hierarchical model with Eigenoptions. As we have demonstrated, learning
eigenfunctions requires large neural networks. Similarly, when learning an
eigenoption, the option policy should be as rich as the spectral network to have a
good approximation of the policy. This has hindered our testing of the HTRPO and
limited our reported results for the RL task.\par

To tackle this problem, the model can be improved by incorporating more
sophisticated parameterization as the sparse weight block used in \cite{Spin}. In
the immediate term, we will pursue more extensive tests of the developed model with
more theoretical examination of the optimization heuristic.

\addcontentsline{toc}{chapter}{Glossary}
\printglossaries
\addcontentsline{toc}{chapter}{References}
\bibliographystyle{unsrtnat}
\bibliography{main}
\appendix

\newpage
\section*{TRPO Implementation Parameters}
\vspace{-1em}
\section*{Chapter 2 proofs}

\begin{lemma}
Let $n\geq 1$ be the number of option policies, and $\pi_g$ the gating policy (distribution over the options set). If $\pi$ and $\tilpi$ are two hierarchical policies such that $\tilpi$ is absolutely continuous w.r.t $\pi$, then at any state $s$, we have
\begin{align}
  \text{D}_{\text{KL}} \lrbrack{ \tilpi(.|s)||\pi(.|s)}
   \leq \text{D}_{\text{KL}}\lrbrack{\tilpi_g(.|s)||\pi_g(.|s)}
    + \sum_{k=1}^{n} \pi_g(o|s)\text{D}_{\text{KL}}\lrbrack{\tilpi(.|s,o))||\pi(.|s,o)}
\end{align}
\end{lemma}
\begin{proof}
\label{thm:proofKL}
We fix a state $s$, and we consider the two distributions p and q:
\begin{align}
\label{eqn:defkl}
\begin{split}
p(a,o) \vcentcolon= \pi(a,o|s)= \pi_g(o|s)\pi(a|s,o),
\\
q(a,o) \vcentcolon= \tilpi(a,o|s) = \tilpi_g(o|s)\tilpi(a|s,o)
\end{split}
\end{align}
and we denote $p_A$ and $q_A$, $p_O$ and $q_O$, respectively, the marginal distributions.

Using the conditional KL-divergence, we have th following two identities
\begin{align}
\text{D}_{\text{KL}} \lrbrack{q||p} = \text{D}_{\text{KL}} \lrbrack{q_A||p_A} + \sum_a q_A(a) \text{D}_{\text{KL}}  \lrbrack{q(.|a)||p(.|a)} \label{KLC1}\\
\text{D}_{\text{KL}} \lrbrack{q||p} = \text{D}_{\text{KL}} \lrbrack{q_O||p_O} + \sum_o q_O(o) \text{D}_{\text{KL}}  \lrbrack{q(.|o)||p(.|o)}\label{KLC2}
\end{align}
Substracting \autoref{KLC2} from \autoref{KLC1}, we obtain:

\begin{align}
\text{D}_{\text{KL}}\lrbrack{q_A||p_A}  =  \text{D}_{\text{KL}} \lrbrack{q_O||p_O} + \sum_o q_O(o) \text{D}_{\text{KL}}  \lrbrack{q(.|o)||p(.|o)}- \sum_a q_A(a) \text{D}_{\text{KL}}  \lrbrack{q(.|a)||p(.|a)}
\end{align}
Since $\text{D}_{\text{KL}}  \lrbrack{q(.|a)||p(.|a)}\geq 0$for all $a \in \cA$, we have:
\begin{align}
\text{D}_{\text{KL}}\lrbrack{q_A||p_A} \leq  \text{D}_{\text{KL}} \lrbrack{q_O||p_O} + \sum_o q_O(o) \text{D}_{\text{KL}}  \lrbrack{q(.|o)||p(.|o)}
\end{align}

All is left is to notice that :
\begin{align}
p&_A = \pi(.|s),\text{ } q_A = \tilpi(.|s) \nonumber \\
p&_o = \pi_g(.|s),\text{ } q_o = \tilpi_g(.|s) \nonumber \\
p&(.|o) = \pi(.|s,o),\text{ } q(.|o) = \tilpi(.|s,o). \nonumber
\end{align}
Hence :
\begin{align}
\text{D}_{\text{KL}}\lrbrack{\tilpi(.|s)||\pi(.|s)} \leq  \text{D}_{\text{KL}} \lrbrack{\tilpi_g(.|s)||\pi_g(.|s)} + \sum_o \tilpi_g(o|s) \text{D}_{\text{KL}}  \lrbrack{\tilpi(.|s,o)||\pi(.|s,o)}
\end{align}
\end{proof}

\end{document}